\documentclass[runningheads]{llncs}
\usepackage[dvipsnames]{xcolor}
\usepackage{graphicx}
\usepackage{amsmath}
\usepackage{amssymb}
\usepackage{mathtools}
\usepackage{todonotes}
\usepackage{nicefrac}
\usepackage[bookmarks,bookmarksopen,bookmarksdepth=2]{hyperref}

\hypersetup{
  colorlinks   = true,    
  urlcolor     = black,    
  linkcolor    = black,    
  citecolor    = black      
}
\usepackage{url}
\usepackage{cleveref}

\usepackage[math]{cellspace}
\usepackage{tikz}
\usepackage{ifthen}
\usetikzlibrary{shapes}
\usetikzlibrary{shapes.misc}

\newcommand{\MinLength}[3]{%
\ifthenelse{\lengthtest{\the#1<\the#2}}%
           {\setlength{#3}{#1}}%
           {\setlength{#3}{#2}}%
}

\newcommand{\MaxLength}[3]{%
\ifthenelse{\lengthtest{\the#1>\the#2}}%
           {\setlength{#3}{#1}}%
           {\setlength{#3}{#2}}%
}

\newlength{\xscale}
\newlength{\yscale}
\newlength{\mscale}

\newcommand{\setstringscale}[2] {%
	\setlength{\xscale}{#1}%
	\setlength{\yscale}{#2}%
	\MinLength{\xscale}{\yscale}{\mscale}%
}

\newcommand{\stringdisplaystyle} {%
	\setstringscale{0.33cm}{0.23cm}%
}
\newcommand{\stringinlinestyle} {%
	\setstringscale{0.2cm}{0.12cm}%
}

\newlength{\boxscale}
\makeatletter 
\newcommand{\autostringstyle}{%
    \if@display
	\stringdisplaystyle
    \else
	\stringinlinestyle
    \fi
}
\makeatother 

\newcommand{\wire}[1][1]{\draw[rounded corners=#1\mscale]}

\newcommand{\fg}[1]{\begin{pgfonlayer}{foreground}#1\end{pgfonlayer}}

\newcommand{\blackdot}[1]{\fg{\draw[fill=black] (#1) circle (0.25\mscale);}}

\newcommand{\sqbox}[3]{
\fg{
	\draw (#1) node[
		fill=white,
		rounded rectangle,
		draw,
		minimum height=#3\boxscale,
		inner sep = 1pt,
		rounded rectangle left arc=none,
		rounded rectangle right arc=none,
		rounded rectangle arc length = 180,
	] {\small #2};}
}

\newcommand{\stoch}[3]{
\fg{
	\draw (#1) node[
		fill=white,
		rounded rectangle,
		draw,
		minimum height=#3\boxscale,
		inner sep = 1pt,
		rounded rectangle left arc=none,
		rounded rectangle arc length = 120
	] {\small #2};
}
}

\newcommand{\ulabel}[2]{\fg{\draw (#1) +(0,5pt) node[inner sep = 0] {\scriptsize $#2$};}}

\newcommand{\stringdiagram}[1]{%
\mbox{
\autostringstyle%
\begin{tikzpicture}[x=\xscale, y=\yscale, baseline={([yshift=-0.45ex]current bounding box.center)}]
\pgfdeclarelayer{grid layer}
\pgfdeclarelayer{foreground}
\pgfsetlayers{grid layer,main,foreground}
#1
\end{tikzpicture}%
}}

\newcommand{\pbox}[2]{
	\left.
	#1
	\right._{\raisebox{0.5\yscale}{$#2$}}
}

\DeclareMathOperator*{\argmax}{arg\,max}

\newcommand{\State}{{\mathcal{M}}} 
\newcommand{\m}{m} 
\newcommand{\In}{{\mathcal{I}}} 
\newcommand{\s}{i} 
\newcommand{\Out}{{\mathcal{O}}} 
\newcommand{\upd}{{\mu}} 
\newcommand{\expose}{{\omega}} 
\newcommand{\Hid}{{\mathcal{H}}} 
\newcommand{\Act}{{\mathcal{A}}} 

\newcommand{\Sen}{{\mathcal{S}}} 
\newcommand{\tran}{{\nu}}  
\newcommand{\rew}{{r}} 

\usepackage{xspace}
\newcommand{\smm}{stochastic Moore machine\xspace}
\newcommand{\smms}{stochastic Moore machines\xspace}

\newcommand{\Ihid}{\Hid}  
\newcommand{\Iout}{\Act} 

\newcommand{\mdpState}{{\mathcal{X}}}  
\newcommand{\mdpAct}{{\mathcal{A}}} 

\newcommand{\btran}{{\beta}} 
\newcommand{\brew}{{\rho}} 

\begin{document}
%
\title{Interpreting systems as solving POMDPs: a step towards a formal understanding of agency\thanks{
This project was made possible through the support of Grant 62229 from the John Templeton Foundation. The opinions expressed in this publication are those of the authors and do not necessarily reflect the views of the John Templeton Foundation.
Work on this project was also supported by a grant from GoodAI.}}

\titlerunning{Interpreting systems as solving POMDPs}
%

\author{Martin Biehl\inst{1}\orcidID{0000-0002-1670-6855} \and Nathaniel Virgo\inst{2}\orcidID{0000-0001-8598-590X}}
%
\authorrunning{M.\ Biehl \and N.\ Virgo}
%

\institute{Cross Labs, Cross Compass, Tokyo 104-0045, Japan\\
\email{martin.biehl@cross-compass.com}\\
\and Earth-Life Science Institute, Tokyo Institute of Technology, Tokyo 152-8550, Japan}
%
%

\maketitle              
\begin{abstract}

Under what circumstances can a system be said to have beliefs and goals, and how do such agency-related features relate to its physical state? Recent work has proposed a notion of \emph{interpretation map}, a function that maps the state of a system to a probability distribution representing its beliefs about an external world. Such a map is not completely arbitrary, as the beliefs it attributes to the system must evolve over time in a manner that is consistent with Bayes' theorem, and consequently the dynamics of a system constrain its possible interpretations. Here we build on this approach, proposing a notion of interpretation not just in terms of beliefs but in terms of goals {and actions}. To do this we make use of the existing theory of partially observable Markov decision processes (POMDPs): 
we say that a system can be interpreted as a solution to a POMDP if it not only admits an interpretation map describing its {beliefs about} the hidden state of a POMDP but also takes actions that are optimal according to its {belief state}. 
An agent is then a system together with an interpretation of this system as a POMDP solution.
Although POMDPs are not the only possible formulation of what it means to have a goal, this nevertheless represents a step towards a more general {formal} definition of what it means for a system to be an agent.



\end{abstract}
\section{Introduction}
This work is a contribution to the general question of when a physical system can justifiably be seen as an agent. 
We are still far from answering this question in full generality but employ here a set of limiting assumptions / conceptual commitments that allow us to provide an \emph{example} of the kind of answer we are looking for. 

The basic idea is inspired by but different from Dennett's proposal to use so-called stances \cite{dennett_true_1981}, which says we should interpret a system as an agent if taking the \emph{intentional stance} improves our predictions of its behavior beyond those obtained by the \emph{physical stance} (or the design stance, but we ignore this stance here).
Taking the physical stance means using the dynamical laws of the (microscopic) physical constituents of the system.
Taking the intentional stance means ignoring the dynamics of the physical constituents of the system and instead interpreting it as a rational agent with beliefs and desires. (We content ourselves with only ascribing \emph{goals} instead of desires.)
A quantitative method to perform this comparison of stances can be found in \cite{orseau_agents_2018}.

In contrast to using a comparison of prediction performance of different stances we propose to decide whether a system can be interpreted as an agent by checking whether 
its physical dynamics
are  \emph{consistent} with an interpretation as a rational agent with beliefs and goals.
%
%
%
In other words, assuming that we know what happens in the system on the physical level (admittedly a strong assumption), we propose to check whether we can consistently ascribe meaning to its
physical
states, such that they appear to implement a process of belief updating and decision making.

A formal example definition of what it means for an interpretation to be consistent was recently published in \cite{virgo_interpreting_2021}. This establishes a notion of consistent interpretation as a Bayesian \emph{reasoner}, meaning something that receives inputs and uses them to make inferences about some hidden variable, but does not take actions or pursue a goal. 

Briefly, such an interpretation consists of a map from the physical / internal states of the system to Bayesian beliefs about hidden states 
(that is, probability distributions over them), as well as a model describing how the hidden states determine the next hidden state and the input to the system. 
To be consistent, if the internal state at time $t$ is mapped to some belief, then the internal state at time $t+1$ must map to the Bayesian posterior of that belief, given the input that was received in between the two time steps. 

In other words, the internal state parameterizes beliefs and the system updates the parameters in a way that makes the parameterized belief change according to Bayes law. 
A Bayesian reasoner is not an agent however. It lacks both goals and rationality since it neither has a goal nor actions that it could rationally take to bring the goal about. 

Here we build on the notion of consistent interpretations of \cite{virgo_interpreting_2021} and show how it can be extended to also include the attribution of goals and rationality.

For this we employ the class of problems called partially observable Markov decision processes (POMDPs), which are well suited to our purpose. These provide hidden states to parameterize beliefs over, a notion of a goal, and a notion of what it means to act optimally, and thus rationally, with respect to this goal. Note that both the hidden states and the goal (which will be represented by rewards) are not assumed to have a physical realization. They are part of the interpretation and therefore only need to exist in the mathematical sense. Informally, the hidden state is assumed by the agent to exist, but need not match a state of the true external world.

We will see that given a pair of a physical system (as modelled by a \smm) and a POMDP it can in principle  be checked whether the system does indeed parameterize beliefs over the hidden states and act optimally with respect to the goal and its beliefs (\cref{def:pomdpinterpretation}). We then say the system can be interpreted as solving the POMDP, and we 
propose to 
call the pair of system and POMDP an agent.
%
%
%
%
%
%
This constitutes an example of a formal definition of a \emph{rational agent with beliefs and goals}. 

To get there however we need to make some conceptual commitments / assumptions that restrict the scope of our definition. Note that we do not make these commitments because we believe they are particularly realistic or useful for the description of real world agents like living organisms, but only because they make it possible to be relatively precise. We suspect that each of these choices has alternatives that lead to other notions of agents. Furthermore, we do not argue that all agents are rational, nor that they all have beliefs and goals. These are properties of the particular notion of agent we define here, but there are certainly other notions of agent that one might want to consider.

The first commitment is with respect to the notion of system.  
Generally, the question of which physical systems are agents may require us to clarify how we obtain a candidate physical system from a causally closed universe and what the type of the resulting candidate physical system is. This can be done by defining what it means to be an individual and / or identifying some kind of boundary. Steps in this direction have been made in the context of cellular automata e.g.\ by \cite{beer_cognitive_2014,biehl_towards_2016} and in the context of stochastic differential equations by \cite{friston_life_2013,friston_free_2022}. 

We here restrict our scope by assuming that the candidate physical system is a stochastic Moore machine (\cref{def:smm}). A stochastic Moore machine has inputs, a dynamic and possibly stochastic internal state, and outputs that deterministically depend on the internal state only. This is far from the most general types of system that could be considered, but it is general enough to represent the digital computers controlling most artificial agents at present. It it also similar to a time and space discretized version of the dynamics of the internal state of the literature on the free energy principle (FEP) \cite{friston_free_2022}. 

Already at this point the reader may expect that the inputs of the Moore machine will play the role of sensor values and the outputs that of actions and this will indeed be the case. Furthermore, the role of the ``physical constituents'' or
physical state (of Dennett's physical stance) will be played by the internal state of the machine and this state will be equipped with a kind of consistent Bayesian interpretation. In other words, it will be parameterizing/determining probabilistic beliefs. This is similar to the role of internal states in the FEP.

For our formal notion of beliefs we commit to probability distributions that are updated in accordance with Bayes law.

The third commitment is with respect to a formal notion of goals and rationality. As already mentioned, for those we employ POMDPs. These provide both a formal notion of goals via expected reward maximization and a formal notion of rational behavior via their optimal policy. 
%
%

Combining these commitments we want to express when exactly a system can be interpreted as a rational agent with beliefs and goals. 


Rational agents take the optimal actions with respect to their goals and beliefs. The convenient feature of POMDPs for our purposes is that the optimal policies are usually expressed as functions of probabilistic beliefs about the hidden state of the POMDP. For this to work, the probabilistic beliefs must be updated correctly according to Bayesian principles. 
%
%
%
It then turns out that these standard solutions for POMDPs can be turned into \smms whose states are the (correctly updated) probabilistic beliefs themselves and whose outputs are the optimal actions.

This has two consequences. 
One is that it seems justified to interpret such \smms as rational agents that have beliefs and goals. 
Another is that there are \smms that solve POMDPs.
Accordingly, our definition of \smms that solve POMDPs (\cref{def:pomdpinterpretation}) applies to these machines. 

In addition to such machines, however, we want to include machines whose states only \emph{parameterize} (and are not equal to) the probabilistic beliefs over hidden states and who output optimal actions.\footnote{These machines are probably equivalent to the sufficient information state processes in \cite[definition 2]{hauskrecht_value-function_2000} but establishing this is beyond the scope of this work.}
We achieve this by employing an adapted notion of a consistent interpretation (\cref{def:influencedfiltering}).
A \smm can then be interpreted as solving a POMDP if it has this kind of consistent interpretation with respect to the hidden state dynamics of the POMDP and outputs the optimal policy.

We also show that the machines obeying our definition are optimal in the same sense as the machines whose states are the correctly updated beliefs, so we find it justified to interpret those machines as rational agents with beliefs and goals as well.

Before we go on to the technical part we want to highlight a few more aspects. 
The first is that the existence of a consistent interpretation (either in terms of filtering or in terms of agents) only depends on the stochastic Moore machine that's being interpreted, and not on any properties of its environment.
This is because a consistent interpretation requires an agent's beliefs and goals to be \emph{consistent}, and this is different from asking whether they are \emph{correct}.
An agent may have the wrong model, in that it doesn't correspond correctly to the true environment.
Its conclusions in this case will be wrong, but its reasoning can still be consistent; see \cite{virgo_interpreting_2021} for further discussion of this point.
In the case of POMDP interpretations this means that the agent's actions only need to be optimal according to its model of the environment, but they might be suboptimal according to the true environment.

This differs from the perspective taken in the original FEP literature  concerned with the question of when a system of stochastic differential equations contain an agent performing approximate Bayesian inference \cite{friston_life_2013,friston_free_2019,parr_markov_2020,da_costa_bayesian_2021,friston_free_2022}.\footnote{The FEP literature includes both publications on how to construct agents that solve problems (e.g.\ \cite{friston_active_2015}) and publications on when a system of stochastic differential equations contain an agent performing approximate Bayesian inference. Only the latter literature addresses a question comparable to the one addressed in the present manuscript.} This literature also interprets a system as modelling hidden state dynamics, but there the model is derived from the dynamics of the actual environment (the so called ``external states''), and hence cannot differ from it.
%
We consider it helpful to be able to make a clear distinction between the agent's model of its environment and its true environment.  
The case where the model is derived from the true environment is an interesting special case of this, but our framework covers the general case as well. 
To our knowledge, the possibility of choosing the model independently from the actual environment in a FEP-like theory was first proposed in \cite{virgo_interpreting_2021}, and has since also appeared in a setting closer to the original FEP one \cite{parr_inferential_2022}.

We will see here (\cref{def:influencedfiltering}) that the independence of model from actual environment extends to actions in some sense. Even a machine without any outputs can have a consistent interpretation modelling an influence 
of the internal state on the hidden state dynamics even though it can't have an influence on the actual environment. Such ``actions'' remain confined to the interpretation.

Another aspect of using consistent interpretations of the internal state and thus the analogue of the physical state / the physical constituents of the system is that it automatically comes with a notion of coarse-graining of the internal state. Since interpretations map the internal state to beliefs but don't need to do so injectively they can include coarse-graining of the state. 

Also note, all our current notions of interpretation in terms of Bayesian beliefs require exact Bayesian updating. This means approximate versions of Bayesian inference or filtering are outside of the scope. This limits the scope of our example definition in comparison with the FEP which, as mentioned, also uses beliefs parameterized by internal states but considers approximate inference. On the other hand this keeps the involved concepts simpler.

Finally, we want to mention that \cite{kenton_discovering_2022} recently proposed an agent discovery algorithm. This algorithm is based on a definition of agents that takes into account the creation process of the system. 
An agent discovery algorithm based on the approach presented here would take as input a machine (\cref{def:machine}) or a stochastic Moore machine (\cref{def:smm}) and try to find a POMDP interpretation (\cref{def:pomdpinterpretation}). The creation process of the machine (system) would not be taken into account. 
This is one distinction between our notion of an agent and that of \cite{kenton_discovering_2022}. A more detailed comparison would be interesting but is beyond the scope of this work. 

The rest of this manuscript presents the necessary formal definitions that allow us to precisely state our example of an agent definition.

\section{Interpreting \smms}

Throughout the manuscript we write $P\mathcal{X}$ for the set of all finitely supported probability distributions over a set $\mathcal{X}$. This ensures that all probability distributions we consider only have a finite set of outcomes that occur with non-zero probability. We can then avoid measure theoretic language and technicalities. For two sets $\mathcal{X},\mathcal{Y}$ a Markov kernel is a function $\zeta:\mathcal{X} \to P\mathcal{Y}$. We write $\zeta(y|x)$ for the probability of $y \in \mathcal{Y}$ according to the probability distribution $\zeta(x) \in P\mathcal{Y}$. If we have a function $f:\mathcal{X}\to \mathcal{Y}$ we sometimes write $\delta_f:\mathcal{X} \to P\mathcal{Y}$ for the Markov kernel with $\delta_{f(x)}(y)$ (which is $1$ if $y=f(x)$ and $0$ else) then defining the probability of $y$ given $x$. 
%

We give the following definition, which is the same as the one used in \cite{virgo_interpreting_2021}, but specialised to the case where update functions map to the set of finitely supported probability distributions and not to the space of all probability distributions.  
\begin{definition}
	\label{def:machine}
A \emph{machine} is a tuple $(\State,\In,\upd)$ consisting of 
    a set $\State$ called \emph{internal state space}; 
    a set $\In$ called \emph{input space};
    and a Markov kernel $\upd:\In \times \State \to P\State$ called \emph{machine kernel}, 
    taking an input $\s \in \In$ and a current machine state $\m \in \State$ to a probability distribution $\upd(\s,\m) \in P\State$ over machine states. 
\end{definition}
The idea is that at any given time the machine has a state $m\in \State$. At each time step it recieves an input $\s \in \In$, and updates stochastically to a new state, according to a probability distirbution specified by the machine kernel.
If we add a function that specifies an output given the machine state we get the definition of a \smm.
\begin{definition}
\label{def:smm}
A \emph{\smm} is a tuple $(\State,\In,\Out,\upd,\expose)$ consisting of
    a machine with internal state space $\State$, input space $\In$, and machine kernel $\upd:\In \times \State \to P\State$; 
    a set $\Out$ called the \emph{output space};
    and a function $\expose:\State \to \Out$ called  expose function taking any machine state $m \in \State$ to an output $\expose(m) \in \Out$. 
\end{definition}

Note that the expose function is an ordinary function and not stochastic. 


We need to adapt the definition of a consistent Bayesian filtering interpretation \cite[Definition 2]{virgo_interpreting_2021}.
For our purposes here we need to include models of dynamic hidden states that can be influenced. 
In particular we need to interpret a machine as modelling the dynamics of a hidden state that the machine itself can influence. 
This suggests that the interpretation includes a model of how the state of the machine influences the hidden state. We here call such influences ``actions'' and the function that takes states to actions \emph{action kernel}. 
\begin{definition}
\label{def:influencedfiltering}
Given a machine with state space $\State$, input space $\In$ and machine kernel $\upd:\In \times \State \to P\State$, a \emph{consistent Bayesian influenced filtering interpretation}  $(\Hid,\Act,\psi,\alpha,\kappa)$ consists of 
    a set $\Hid$ called the \emph{hidden state space}; 
    a set $\Act$ called the \emph{action space};
    a Markov kernel $\psi:\State \to P\Hid$ called \emph{interpretation map} mapping machine states to probability distributions over the hidden state space;
    a function $\alpha:\State \to \Act$ called \emph{action function} mapping machine states to actions\footnote{We choose actions to be deterministic functions of the machine state because the \smms considered here also have deterministic outputs. Other choices may be more suitable in other cases.}; 
    and a Markov kernel $\kappa:\Hid \times \Act \to P(\Hid \times \In)$ called the \emph{model kernel} mapping pairs $(h,a)$ of hidden states and actions to probability distributions $\kappa(h,a)$ over pairs $(h',i)$ of next hidden states and an input. 

These components have to obey the following equation. First, in string diagram notation (see appendix A of \cite{virgo_interpreting_2021} for an introduction to string diagrams for probability in a similar context to the current paper):
\begin{equation}
%
%
\stringdiagram {
 \wire (-1,0) -- (12.5,0);
 \ulabel{-1,0}{\State};
 \ulabel{12.5,0}{\State};
 \blackdot{0,0};
 \wire[1.5] (0,0) -- (0,5) -- (1.75,5);
 \stoch{1.75,5}{$\psi$}{3};
 \blackdot{0.5,0};
 \wire[1] (0.5,0) -- (0.5,2) -- (1.75,2);
 \sqbox{1.75,2}{$\alpha$}{3};
 \ulabel{12.5,2}{\In};
 \wire (1.75,5) -- (6,5);
 \wire (1.75,2) -- (6,2);
 \ulabel{3.2,5}{\Ihid};
 \ulabel{3.2,2}{\Iout};
 \stoch{6,3.5}{$\kappa$}{8};
 \wire (6,5) -- (12.5,5);
 \wire (6,2) -- (12.5,2);
 \stoch{10.5,0.25}{$\upd$}{3};
 \blackdot{8.5,2};
 \ulabel{9,2}{\In};
 \wire[1] (8.5,2) -- (8.5,0.5) -- (10.5,0.5);
 \ulabel{12.5,5}{\Ihid};
}
\,\, = 
\pbox{
\stringdiagram {
 \wire (-1,0) -- (14.5,0);
 \ulabel{-1,0}{\State};
 \ulabel{14.5,0}{\State};
 \blackdot{0,0};
 \wire[1.5] (0,0) -- (0,5) -- (1.75,5);
 \stoch{1.75,5}{$\psi$}{3};
 \blackdot{0.5,0};
 \wire[1] (0.5,0) -- (0.5,2) -- (1.75,2);
 \sqbox{1.75,2}{$\alpha$}{3};
 \ulabel{14.5,2}{\In};
 \wire (1.75,5) -- (6,5);
 \wire (1.75,2) -- (6,2);
 \ulabel{3.2,5}{\Ihid};
 \ulabel{3.2,2}{\Iout};
 \stoch{6,3.5}{$\kappa$}{8};
 \wire (6,5) -- (8.5,5);
 \wire (6,2) -- (14.5,2);
 \blackdot{8.5,5};
 \blackdot{8.5,2};
 \ulabel{9,2}{\In};
 \wire[1] (8.5,2) -- (8.5,0.5) -- (10.5,0.5);
 \stoch{10.5,0.25}{$\upd$}{3};
 \blackdot{11.75,0};
 \wire[1] (11.75,0) -- (11.75,5) -- (13,5);
 \stoch{13,5}{$\psi$}{3};
 \wire (13,5) -- (14.5,5);
 \ulabel{14.5,5}{\Ihid};
}
}{,}
\end{equation}
Second, in more standard notation, we must have for each $m \in \State$, $h' \in \Ihid$, $i \in \In$, and $m' \in \State$:
\begin{align}
\label{eq:influencedtradconsistency}
    \begin{split}
        \left(\sum_{h \in \Ihid}\right.&\left.\sum_{a \in \Iout}\kappa(h',i|h,a) \psi(h|m)\delta_{\alpha(m)}(a) \right) \upd(m'|i,m)=\\
        &\psi(h'|m') \left(\sum_{h \in \Ihid}\sum_{a \in \Iout} \sum_{ h'' \in \Ihid}\kappa( h'',i|h,a) \psi(h|m)\delta_{\alpha(m)}(a)\right) \upd(m'|i,m).
    \end{split} 
\end{align}
\end{definition}
In \cref{app:familiarconsistency} we show how to turn \cref{eq:influencedtradconsistency} into a more familiar form.



Note that we defined consistent Bayesian influenced filtering interpretations for machines that have no actual output but that it also applies to those with outputs. If we want an interpretation of a machine with outputs we may choose the action space as the output space and the action kernel as the output kernel, but we don't have to. Interpretations can still be consistent.    

Also note that when $\Iout$ is a space with only one element we recover the original definition of a consistent Bayesian filtering interpretation from \cite{virgo_interpreting_2021}.

\section{Interpreting \smms as solving POMDPs}

\begin{definition}
\label{def:pomdp}
A \emph{partially observable Markov decision process} (POMDP) can be defined as a tuple $(\Hid,\Act,\Sen,\kappa,\rew)$ consisting of 
    a set $\Hid$ called the \emph{hidden state space}; 
    a set $\Act$ called the \emph{action space};
    a set $\Sen$ called the \emph{sensor space};
    a Markov kernel $\kappa:\Hid \times \Act \to P(\Hid\times \Sen)$ called the transition kernel taking a hidden state $h$ and action $a$ to a probability distribution over next hidden states and sensor values;
    and a function $\rew:\Hid \times \Act \to \mathbb{R}$ called the \emph{reward function} returning a real valued reward depending on the hidden state and an action.
\end{definition}

To solve a POMDP we have to choose a policy (as defined below) that maximizes the expected cumulative reward either for a finite horizon or discounted with an infinite horizon. We only deal with the latter case here.

POMDPs are commonly solved in two steps. First since the hidden state is unknown, probability distributions $b \in P\Hid$ (called belief states) over the hidden state are introduced and an updating function $f:P\Hid \times \Act \times \Sen \to P\Hid$ for these belief states is defined. This updating is directly derived from Bayes rule \cite{kaelbling_planning_1998}:
    \begin{align}
    \label{eq:beliefupdate}
        b'(h')=f(b,a,s)(h')\coloneqq&Pr(h'|b,a,s)
        \coloneqq&\frac{\sum_{h \in \Hid}\kappa(h',s|h,a)b(h)}{\sum_{\bar h,\bar h' \in \Hid} \kappa(\bar h',s|\bar h,a)b(\bar h)}.
    \end{align}
(Note that an assumption is that the denominator is greater than zero.)
Then an optimal policy $\pi^*:P\Hid \to \Act$ mapping those belief states to actions is derived from a so-called \emph{belief state MDP} (see \cref{app:beliefmdp} for details). 
The optimal policy can be expressed using an optimal value function
$V^*:P\Hid \to \mathbb{R}$ that solves the following \emph{Bellman equation} \cite{hauskrecht_value-function_2000}:
\begin{align}
\label{eq:valuefun}
    V^*(b)=&\max_{a \in \mdpAct} \left(\sum_{h \in \Hid} b(h) r(h,a) + \gamma \hspace{-3pt}\sum_{\substack{s \in \Sen\\ h,h' \in \Hid}} \hspace{-3pt} \kappa(h',s|h,a)b(h) V^*(f(b,a,s)) \right).
\end{align}
The optimal policy is then \cite{hauskrecht_value-function_2000}:
\begin{align}
\label{eq:optimalpolicy}
\pi^*(b)=&\argmax_{a \in \mdpAct} \left(\sum_{h \in \Hid} b(h) r(h,a) + \gamma \hspace{-3pt}\sum_{\substack{s \in \Sen\\ h,h' \in \Hid}} \hspace{-3pt}\kappa(h',s|h,a)b(h) V^*(f(b,a,s)) \right).
\end{align}
Note that the belief state update function $f$ determines optimal value function and policy. 

Define now $f_{\pi^*}(b,s)\coloneqq f(b,\pi^*(b),s)$.
Then note that if we consider $P\Hid$ a state space, $\Sen$ an input space, $\Act$ an output space, $\delta_{f_{\pi^*}}:P\Hid \times \Sen \to PP\Hid$  a machine kernel, and $\pi^*:P\Hid \to \Act$ an expose kernel, we get a stochastic Moore machine.\footnote{If the denominator in \cref{eq:beliefupdate} is zero for some value $s \in \Sen$ then define e.g.\ $f_{\pi^*}(b,s)=b$.}

This machine solves the POMDP and can be directly interpreted as a rational agent with beliefs and a goal. The beliefs are just the belief states themselves, the goal is expected cumulative reward maximization, and the optimal policy ensures it acts rationally with respect to the goal.

Our definition of interpretations of \smms as solutions to POMDPs includes this example and extends it to machines whose states aren't probability distributions / belief states directly but instead are parameters of such belief states that get (possibly stochastically) updated consistently.  

We now state this main definition and then a proposition that ensures that our definition only applies to \smms that parameterize beliefs correctly as required by \cref{eq:beliefupdate}. This ensures that the optimal policy obtained via \cref{eq:optimalpolicy} is also the optimal policy for the states of the machine.

\begin{definition}
\label{def:pomdpinterpretation}
Given a \smm $(\State,\In,\Out,\upd,\expose)$, a consistent interpretation as a solution to a POMDP  is given by a POMDP $(\Hid,\Out,\In,\kappa,\rew)$ and an \emph{interpretation map} $\psi:\State \to P\Hid$ 
such that 
    (i) $(\Hid,\Out,\psi,\expose,\kappa)$ is a consistent Bayesian influenced filtering interpretation of the machine part  $(\State,\In,\upd)$ of the \smm; %
%
    %
and $(ii)$ the machine expose function $\expose:\State \to \Out$ (which coincides with the action function in the interpretation) maps any machine state $m$ 
    to the action $\pi^*(\psi(m))$ specified by the optimal POMDP policy for the belief $\psi(m)$ associated to machine state $m$ by the interpretation.
    Formally: 
    \begin{align}
    \expose(m)=& \pi^*(\psi(m)).
    \end{align}
\end{definition}
Note that the machine never gets to observe the rewards of the POMDP we use to interpret it.
An example of a \smm together with an interpretation of it as a solution to a POMDP is given in \cref{app:sondikexample}.


\begin{proposition}
\label{prop:update}
Consider a \smm $(\State,\In,\Out,\upd,\expose)$, together with a consistent interpretation as a solution to a POMDP, given by the POMDP $(\Hid,\Out,\In,\kappa,\rew)$ and Markov kernel $\psi:\State \to P\Hid$.
Suppose it is given an input $i\in \In$, and that this input has a positive probability according to the interpretation. (That is, \cref{eq:subjectivelypossible} is obeyed.)
Then the parameterized distributions $\psi(m)$ update as required by the belief state update equation (\cref{eq:beliefupdate}) whenever $a=\pi^*(b)$ i.e.\ whenever the action is equal to the optimal action. More formally, for any $m,m' \in \State$ with $\upd(m'|i,m)>0$ and $i \in \In$ that can occur according to the POMDP transition and sensor kernels, we have for all $h' \in \Hid$
\begin{align}
    \psi(h'|m') = f(\psi(m),\pi^*(\psi(m)),i)(h').
\end{align}
\end{proposition}
\begin{proof}
See \cref{app:updateproof}.
\end{proof}
With this we can see that if $V^*$ is the optimal value function for belief states $b \in P\Hid$ of \cref{eq:valuefun}, then $\bar V^*(m) \coloneqq V^*(\psi(m))$ is an optimal value function on the machine's state space with optimal policy $\expose(m)=\pi^*(\psi(m))$.
\section{Conclusion}
We proposed a definition of when an \smm can be interpreted as solving a partially observable Markov decision process (POMDP). We showed that standard solutions of POMDPs have counterpart machines that this definition applies to. Our definition employs a newly adapted version of a consistent interpretation. We showed that with this our definition includes additional machines whose state spaces are parameters of probabilistic beliefs and not such beliefs directly. We suspect these machines are closely related to information state processes \cite{hauskrecht_value-function_2000} but the precise relation is not yet known to us.

\bibliographystyle{splncs04}
\bibliography{./bibliography}

\appendix 

\section{Consistency in more familiar form}
\label{app:familiarconsistency}
One way to turn \cref{eq:influencedtradconsistency} into a probably more familiar form is to introduce some abbreviations and look at some special cases. We follow a similar strategy to~\cite{virgo_interpreting_2021}. Let
\begin{align}
    \psi_{\Ihid,\In}(h',i|m)\coloneqq\sum_{h \in \Ihid}\sum_{a \in \Iout}\kappa(h',i|h,a) \psi(h|m)\delta_{\alpha(m)}(a)
\end{align}
and 
\begin{align}
    \psi_{\In}(i|m)\coloneqq\sum_{h' \in \Ihid}\psi_{\Ihid,\In}(h',i|m).
\end{align}
Then consider the case of a deterministic machine and choose the $m' \in \State$ that actually occurs for a given input $i \in \In$ such that $\upd(m'|i,m)=1$ or abusing notation $m' = m'(i,m)$. Then we get from \cref{eq:influencedtradconsistency}:
\begin{align}
\label{eq:possiblenextstateconsistency}
            \psi_{\Ihid,\In}(h',i|m)=\psi(h'|\upd(i,m)) \psi_{\In}(i|m).
\end{align}
If we then also consider an input $i \in \In$ that is \emph{subjectively possible} as defined in \cite{virgo_interpreting_2021} which here means that $\psi_{\In}(i|m)>0$ we get 
\begin{align}
    \psi(h'|m'(i,m))= \frac{\psi_{\Ihid,\In}(h',i|m)}{\psi_{\In}(i|m)}.
\end{align}
This makes it more apparent that in the interpretation the updated machine state $m'=m'(i,m)$ parameterizes a belief $\psi(h'|m'(i,m))$ which is equal to the posterior distribution over the hidden state given input $i$. 
In the non-deterministic case, note that when $\upd(m'|i,m)=0$ the consistency equation imposes no condition, which makes sense since that means the machine state $m'$ can never occur. When $\upd(m'|i,m)>0$ we can divide \cref{eq:influencedtradconsistency} by this to also get \cref{eq:possiblenextstateconsistency}. The subsequent argument for $m'=m'(i,m)$ then must hold not only for this one possible next state but instead for every $m'$ with $\upd(m'|i,m)$. So in this case (if $s$ is subjectively possible) any of the possible next states will parameterize a belief $\psi(h'|m')$ equal to the posterior. 

\section{Proof of \cref{prop:update}}
\label{app:updateproof}


For the readers's convenience we recall the proposition:
\begin{proposition}
	Consider a \smm $(\State,\In,\Out,\upd,\expose)$, together with a consistent interpretation as a solution to a POMDP, given by the POMDP $(\Hid,\Out,\In,\kappa,\rew)$ and Markov kernel $\psi:\State \to P\Hid$.
	Suppose it is given an input $i\in \In$, and that this input has a positive probability according to the interpretation. (That is, \cref{eq:subjectivelypossible} is obeyed.)
	Then the parameterized distributions $\psi(m)$ update as required by the belief state update equation (\cref{eq:beliefupdate}) whenever $a=\pi^*(b)$ i.e.\ whenever the action is equal to the optimal action. More formally, for any $m,m' \in \State$ with $\upd(m'|i,m)>0$ and $i \in \In$ that can occur according to the POMDP transition and sensor kernels, we have for all $h' \in \Hid$
	\begin{align}
		\psi(h'|m') = f(\psi(m),\pi^*(\psi(m)),i)(h').
	\end{align}
\end{proposition}
\begin{proof}
By assumption the machine part $(\State,\In,\upd)$ of the \smm has a consistent Bayesian influenced filtering interpretation $(\Hid,\Out,\psi,\expose,\kappa)$.

This means that the belief $\psi(m)$ parameterized by the \smm obeys \cref{eq:influencedtradconsistency}. This means that for every possible next state $m'$ (i.e.\ $\upd(m'|s,m)>0$) we have 
\begin{align}
\label{eq:influencedtradconsistencyprop}
    \begin{split}
        \sum_{h \in \Ihid}&\sum_{a \in \Iout}\kappa(h',i|h,a) \psi(h|m)\delta_{\omega(m)}(a) =\\
        &\psi(h'|m') \left(\sum_{h \in \Ihid}\sum_{a \in \Iout} \sum_{ h'' \in \Ihid}\kappa( h'',i|h,a) \psi(h|m)\delta_{\omega(m)}(a)\right)
    \end{split} 
\end{align}
and for every subjectively possible input, that is, for every input $i \in \In$ with
\begin{align}
\label{eq:subjectivelypossible}
    \sum_{h \in \Ihid}\sum_{a \in \Iout} \sum_{ h'' \in \Ihid}\kappa( h'',i|h,a) \psi(h|m)\delta_{\omega(m)}(a)
    >0
\end{align}
(see below for a note on why this assumption is reasonable) we will have:
\begin{align}
        \psi(h'|m') =&\frac{  \sum_{h \in \Ihid}\sum_{a \in \Iout}\kappa(h',i|h,a) \psi(h|m)\delta_{\omega(m)}(a) }{
         \sum_{h \in \Ihid}\sum_{a \in \Iout} \sum_{ h'' \in \Ihid}\kappa( h'',i|h,a) \psi(h|m)\delta_{\omega(m)}(a)}\\
         =&\frac{  \sum_{h \in \Ihid}\kappa(h',i|h,\omega(m)) \psi(h|m) }{
         \sum_{h \in \Ihid}\sum_{ h'' \in \Ihid}\kappa( h'',i|h,\omega(m)) \psi(h|m)}.
\end{align}
Now consider the update function for which the optimal policy is found \cref{eq:beliefupdate}:
    \begin{align}
        f(b,a,s)(h')
        \coloneqq&\frac{\sum_{h \in \Hid}\kappa(h',s|h,a)b(h)}{\sum_{\bar h,\bar h' \in \Hid} \kappa(\bar h',s|\bar h,a)b(\bar h)}
    \end{align}
and plug in the belief $b=\psi(m)$ parameterized by the machine state, the optimal action $\pi^*(\psi(m))$ specified for that belief by the optimal policy $\pi^*$, and the $s=i$:
    \begin{align}
        f(\psi(m),\pi^*(m),i)(h')
        \coloneqq&\frac{\sum_{h \in \Hid}\kappa(h',i|h,\pi^*(\psi(m)))\psi(m)(h)}{\sum_{\bar h,\bar h' \in \Hid} \kappa(\bar h',i|\bar h,\pi^*(\psi(m)))\psi(m)(\bar h)}.
    \end{align}
Also introduce $\kappa$ and write $\psi(h|m)$ for $\psi(m)(h)$ as usual
    \begin{align}
        f(\psi(m),\pi^*(m),i)(h')
        \coloneqq&\frac{\sum_{h \in \Hid}\kappa(h',i|h,\pi^*(\psi(m)))\psi(h|m)}{\sum_{\bar h,\bar h' \in \Hid} \kappa(\bar h',i|\bar h,\pi^*(\psi(m)))\psi(\bar h|m)}\\
        =&\psi(h'|m').
    \end{align}
Which is what we wanted to prove.
\end{proof}
Note that if \cref{eq:subjectivelypossible} is not true and the probability of an input $i$ is impossible according to the POMDP transition function, the kernel $\psi$, and the optimal policy $\omega$ then \cref{eq:influencedtradconsistencyprop} puts no constraint on the machine kernel $\upd$ since both sides are zero. So the behavior of the \smm in this case is arbitrary. This makes sense since according to the POMDP that we use to interpret the machine this input is impossible, so our interpretation should tell us nothing about this situation.


\section{Sondik's example}
\label{app:sondikexample}
We now consider the example from \cite{sondik_optimal_1978}. This has a known optimal solution. We constructed a \smm from this solution which has an interpretation as a solution to Sondik's POMDP. This proves existence of \smms with such interpretations.

Consider the following \smm:
\begin{itemize}
    \item State space $\State \coloneqq [0,1]$. (This state will be interpreted as the belief probability of the hidden state being equal to $1$.) 
    \item input space $\In=\{1,2\}$
    \item machine kernel $\upd:\In \times \State \to P\State$ defined by deterministic function $g:\In \times \State \to \State$:
    \begin{align}
        \upd(m'|s,m)\coloneqq \delta_{g(s,m)}(m')
    \end{align}
    where
    \begin{align}
         g(S=1,m)\coloneqq \begin{cases}
         \frac{15}{6 m+20}-\frac{1}{2} &\text{ if } 0 \leq m \leq 0.1188\\
         \frac{9}{5} - \frac{72}{5 m+60} &\text{ if } 0.1188 \leq m \leq 1.
        \end{cases}
    \end{align}
    and 
    \begin{align}
         g(S=2,m)\coloneqq \begin{cases}
         2+\frac{20}{3 m-15} &\text{ if } 0 \leq m \leq 0.1188\\
         - \frac{1}{5} - \frac{12}{5 m-40} &\text{ if } 0.1188 \leq m \leq 1.
        \end{cases}
    \end{align}
    \item output space $\Out \coloneqq \{1,2\}$
    \item expose kernel $\omega:\State \to \Out$ defined by
    \begin{align}
        \omega(m)\coloneqq \begin{cases}
        1 \text{ if } 0 \leq m < 0.1188\\
        2 \text{ if } 0.1188 \leq m \leq 1.
        \end{cases}
    \end{align}
\end{itemize}
A consistent interpretation as a solution to a POMDP for this \smm is given by
\begin{itemize}
    \item The POMDP with
    \begin{itemize}
        \item state space $\Hid\coloneqq \{1,2\}$
        \item action space equal to the output space $\Out$ of the machine above
        \item sensor space equal to the input space $\In$ of the machine above
        \item model kernel $\kappa:\Hid \times \Out \to \Hid \times \In$ defined by
        \begin{align}
            \kappa(h',s|h,a)\coloneqq \nu(h'|h,a)\phi(s|h',a)
        \end{align}
        where $\nu:\Hid \times \Out \to P\Hid$ and $\phi:\Hid \times \Out \to P\In$ are shown in \cref{tab:sondik}
        \begin{table}[t]
            \centering
            \begin{tabular}{c|c|c|c}
                Action $a \in \Out$ & $\nu(h'|h,A=a)$ & $\phi(s|h',A=a)$ & $r(h,A=a)$ \\
                \hline
                \hline
                $1$ & $\begin{pmatrix}
              \nicefrac{1}{5} & \nicefrac{1}{2}\\
              \nicefrac{4}{5} & \nicefrac{1}{2}
            \end{pmatrix}$ & $\begin{pmatrix}
              \nicefrac{1}{5} & \nicefrac{3}{5}\\
              \nicefrac{4}{5} & \nicefrac{2}{5}
            \end{pmatrix}$
                 & $\begin{pmatrix}
              4 \\
              -4 
            \end{pmatrix}$\\
                \hline
                $2$ & $\begin{pmatrix}
              \nicefrac{1}{2} & \nicefrac{2}{5}\\
              \nicefrac{1}{2} & \nicefrac{3}{5}
            \end{pmatrix}$ & $\begin{pmatrix}
              \nicefrac{9}{10} & \nicefrac{2}{5}\\
              \nicefrac{1}{10} & \nicefrac{3}{5}
            \end{pmatrix}$
                 & $\begin{pmatrix}
              0 \\
              -3 
            \end{pmatrix}$\\
            \end{tabular}
            \caption{Sondik's POMDP data.}
            \label{tab:sondik}
        \end{table}
        \item reward function $r:\Hid \times \Out \to \mathbb{R}$ also shown in \cref{tab:sondik}.
    \end{itemize}
    \item Markov kernel $\psi : \State \to P\Hid$ given by:
    \begin{align}
        \psi(h|m):=m^{\delta_1(h)}(1-m)^{\delta_2(h)}.
    \end{align}
\end{itemize}
To verify this we have to check that $(\Hid,\Out,\psi,\expose,\kappa)$ is a consistent Bayesian influenced filtering interpretation of the machine $(\State,\In,\upd)$. For this we need to check \cref{eq:influencedtradconsistency} with $\delta_{\alpha(m)}(a)\coloneqq \delta_{\omega(m)}(a)$. So for each each $m \in [0,1]$, $h' \in \{1,2\}$, $i \in \{1,2\}$, and $m' \in [0,1]$ we need to check:
\begin{align}
    \begin{split}
        \left(\sum_{h \in \Ihid}\right.&\left.\sum_{a \in \Iout}\kappa(h',i|h,a) \psi(h|m)\delta_{\omega(m)}(a) \right) \upd(m'|i,m)=\\
        &\psi(h'|m') \left(\sum_{h \in \Ihid}\sum_{a \in \Iout} \sum_{ h'' \in \Ihid}\kappa( h'',i|h,a) \psi(h|m)\delta_{\omega(m)}(a)\right) \upd(m'|i,m).
    \end{split} 
\end{align}
This is tedious to check but true. We would usually also have to show that $\omega$ is indeed the optimal policy for Sondik's POMDP but this is shown in \cite{sondik_optimal_1978}.


\section{POMDPs and belief state MDPs}
\label{app:beliefmdp}
Here we give some more details about belief state MDPs and the optimal value function and policy of \cref{eq:valuefun,eq:optimalpolicy}. There is no original content in this section and it follows closely the expositions in \cite{hauskrecht_value-function_2000,kaelbling_planning_1998}.

We first define an MDP and its solution and then discuss then add some details about the belief state MDP associated to a POMDP.
\begin{definition}
A \emph{Markov decision process} (MDP) can be defined as a tuple $(\mdpState,\mdpAct,\tran,\rew)$ consisting of
    a set $\mdpState$ called the \emph{state space},
    a set $\mdpAct$ called the \emph{action space},
    a Markov kernel $\tran:\mdpState \times \mdpAct \to P(\mdpState)$ called the \emph{transition kernel}, 
    and a \emph{reward function} $\rew:\mdpState \times \mdpAct \to \mathbb{R}$. 
    Here, the transition kernel takes a state $x \in \mdpState$ and an action $a \in \mdpAct$ to a probability distribution $\tran(x,a)$ over next states and the reward function returns a real-valued instantaneous reward $r(x,a)$ depending on the hidden state and an action. 
\end{definition}

A solution to a given MDP is a control policy. 
As the goal of the MDP we here choose the maximization of expected cumulative discounted reward for an infinite time horizon (an alternative would be to consider finite time horizons). 
This means an optimal policy maximizes 
\begin{align}
\label{eq:totalreward}
    \mathbb{E}\left[\sum_{t=1}^\infty \gamma^{t-1} \rew(x_t,a_t)\right].
\end{align}
where $0<\gamma<1$ is a parameter called the discount factor. This specifies the goal.

To express the optimal policy explicitly we can use the optimal value function $V^*:\mdpState \to \mathbb{R}$. This is the solution to the Bellman equation \cite{kaelbling_planning_1998}:
\begin{align}
    V^*(x)=&\max_{a \in \mdpAct} \left(\rew(x,a) + \gamma \sum_{x' \in \mdpState} \tran(x'|a,x) V^*(x') \right).
\end{align}
The optimal policy is then the function $\pi^*:\mdpState \to \mdpAct$ that greedily maximizes the optimal value function  \cite{kaelbling_planning_1998}: 
\begin{align}
    \pi^*(x)=&\argmax_{a \in \mdpAct} \left(\rew(x,a) + \gamma \sum_{x' \in \mdpState} \tran(x'|a,x) V^*(x') \right).
\end{align}

\subsection{Belief state MDP}

The belief state MDP for a POMDP (see \cref{def:pomdp}) is defined using the belief state update function of \cref{eq:beliefupdate}. We first define this function again here with an additional intermediate step:
    \begin{align}
        f(b,a,s)(h')\coloneqq& Pr(h'|b,a,s)\\
        =&\frac{Pr(h',s|b,a)}{Pr(s|b,a)}\\
        =&\frac{\sum_{h \in \Hid}\kappa(h',s|h,a)b(h)}{\sum_{\bar h,\bar h' \in \Hid} \kappa(\bar h',s|\bar h,a)b(\bar h)}.
    \end{align}
The function $f(b,a,s)$ returns the posterior belief over hidden states $h$ given prior belief $b \in P\Hid$, an action $a \in \mdpAct$ and observation $s \in \Sen$. The Markov kernel $\delta_{f}:P\Hid \times \Sen \times \Act \to PP\Hid$ associated to this function can be seen as a probability of the next belief state $b'$ given current belief state $b$, action $a$ and sensor value $s$:
\begin{align}
    Pr(b'|b,a,s) =\delta_{f(b,a,s)}(b').
\end{align}
Intuitively, the belief state MDP has as its transition kernel the probability $Pr(b'|b,a)$ expected over all next sensor values of the next belief state $b'$ given that the current belief state is $b$ the action is $a$ and beliefs get updated according to the rules of probability, so
\begin{align}
    Pr(b'|b,a)=& \sum_s Pr(b'|b,a,s) Pr(s|b,a)\\
    =& \sum_{s \in \Sen} \delta_{f(b,a,s)}(b') \sum_{h,h' \in \Hid} \kappa(h',s|h,a)b(h).
\end{align}
This gives some intuition behind the definition of belief state MDPs.
\begin{definition}
Given a POMDP $(\Hid,\Act,\Sen,\kappa,\rew)$ the \emph{associated belief state Markov decision process} (belief state MDP) is the MDP $(P\Hid,\Act,\btran,\brew)$ where
\begin{itemize}
    \item the state space $P\Hid$ is the space of probability distributions \emph{beliefs} over the hidden state of the POMDP. We write $b(h)$ for the probability of a hidden state $h \in \Hid$ according to belief $b \in P\Hid$.
    \item the action space $\mathcal{A}$ is the same as for the underlying POMDP
    \item the transition kernel $\kappa: P\Hid \times A \to P\Hid$ is defined as \cite[Section 3.4]{kaelbling_planning_1998} 
    \begin{align}
      \btran(b'|b,a)\coloneqq&
      \sum_{s \in \Sen} \delta_{f(b,a,s)}(b') \sum_{h,h' \in \Hid} \kappa(h',s|h,a)b(h).
    \end{align}
    \item the reward function $\brew:P\Hid \times \mdpAct \to \mathbb{R}$ is defined as
    \begin{align}
        \brew(b,a)\coloneqq\sum_{h \in \Hid} b(h) \rew(h,a).
    \end{align}
    So the reward for action $a$ under belief $b$ is equal to the expectation under belief $b$ of the original POMDP reward of that action $a$.
\end{itemize}
\end{definition}

\subsection{Optimal belief-MDP policy}

Using the belief MDP we can express the optimal policy for the POMDP.

The optimal policy can be expressed in terms of the \emph{optimal value function of the belief MDP}. This is the solution to the equation \cite{hauskrecht_value-function_2000}
\begin{align}
    V^*(b)=&\max_{a \in \mdpAct} \left(\brew(b,a) + \gamma \sum_{b' \in P\Hid} \btran(b'|a,b) V^*(b') \right)\\
    V^*(b)=&\max_{a \in \mdpAct} \left(\brew(b,a) + \gamma \sum_{b' \in P\Hid} \sum_{s \in \Sen} \delta_{f(b,a,s)}(b') \sum_{h,h' \in \Hid} \kappa(h',s|h,a)b(h) V^*(b') \right)\\
    V^*(b)=&\max_{a \in \mdpAct} \left(\brew(b,a) + \gamma \sum_{s \in \Sen} \sum_{h,h' \in \Hid} \kappa(h',s|h,a)b(h) V^*(f(b,a,s)) \right).
\end{align}
This is the expression we used in \cref{eq:valuefun}.
The optimal policy for the belief MDP is then \cite{hauskrecht_value-function_2000}:
\begin{align}
\pi^*(b)=&\argmax_{a \in \mdpAct} \left(\brew(b,a) + \gamma \sum_{s \in \Sen} \sum_{h,h' \in \Hid} \kappa(h',s|h,a)b(h) V^*(f(b,a,s)) \right).
\end{align}
This is the expression we used in \cref{eq:optimalpolicy}.

\end{document}